\documentclass{article}
\usepackage{kr}

\usepackage{times}
\usepackage{soul}
\usepackage{url}
\usepackage[hidelinks]{hyperref}
\usepackage[utf8]{inputenc}
\usepackage[small]{caption}
\usepackage{graphicx}
\usepackage{amsmath}
\usepackage{amsthm}
\usepackage{booktabs}
\usepackage{algorithm}
\usepackage{algorithmic}
\usepackage{graphicx}
\usepackage{booktabs}       
\usepackage{xcolor}         
\usepackage{color}
\usepackage[noabbrev,nameinlink,capitalise]{cleveref}
\usepackage[aboveskip=2pt,belowskip=2pt]{subcaption}
\usepackage{comment}
\usepackage{multirow,setspace}
\usepackage{adjustbox}
\usepackage{amsfonts}       
\usepackage{amsmath,amssymb}
\usepackage{amsthm,latexsym}
\usepackage{bm,amsbsy,isomath}
\usepackage{enumitem}
\usepackage{titlesec}
\usepackage{relsize}
\usepackage[normalem]{ulem}
\usepackage{xfrac}
\usepackage{nicefrac}       
\usepackage{etoolbox}
\usepackage{microtype}      
\usepackage{stmaryrd}

% ...

\usepackage{xspace}
\usepackage{adjustbox}         
\usepackage{multirow,siunitx}
\usepackage{caption}
\usepackage{comment}

\usepackage{pgf}
\usepackage{tikz}
\usepackage{forest}

\usetikzlibrary{arrows,decorations.pathmorphing,decorations.footprints,fadings,calc,trees,mindmap,shadows,decorations.text,patterns,positioning,shapes,matrix,fit}
\usetikzlibrary{graphs}
\usetikzlibrary{decorations.pathreplacing}
\usetikzlibrary{arrows,shadows,backgrounds}
\usetikzlibrary{arrows.meta}
\usetikzlibrary{positioning,fit}
\usetikzlibrary{automata}
\usetikzlibrary{matrix}
\usetikzlibrary{karnaugh}
\usetikzlibrary{shapes.symbols,shapes.misc,shapes.arrows}
\usetikzlibrary{matrix,chains,scopes,decorations.pathmorphing}

\newtheorem{proposition}{Proposition}
\newtheorem{definition}{Definition}

\newtheorem{example}{Example}

\crefname{theorem}{Theorem}{Theorems}
\crefname{lemma}{Lemma}{Lemmas}
\crefname{proposition}{Proposition}{Propositions}
\crefname{definition}{Definition}{Definitions}
\crefname{corollary}{Corollary}{Corollaries}
\crefname{example}{Example}{Examples}
\crefname{claim}{Claim}{Claims}
\crefname{assumption}{Assumption}{Assumptions}

\newcommand{\fml}[1]{{\mathcal{#1}}}

\newcommand{\tn}[1]{\textnormal{#1}}
\newcommand{\tbf}[1]{\textbf{#1}}
\newcommand{\mbf}[1]{\ensuremath\mathbf{#1}}

\newcommand{\mbb}[1]{\ensuremath\mathbb{#1}}
\newcommand{\mrm}[1]{\ensuremath\mathrm{#1}}

\newcommand{\waxp}{\ensuremath\mathsf{WeakAXp}}
\newcommand{\wcxp}{\ensuremath\mathsf{WeakCXp}}
\newcommand{\axp}{\ensuremath\mathsf{AXp}}
\newcommand{\cxp}{\ensuremath\mathsf{CXp}}

\newcommand{\osm}{\textsf{{\small{one-step method}}}\xspace}
\newcommand{\tsm}{\textsf{{\small{two-step method}}}\xspace}

\newcommand{\stwop}{\mrm{\Sigma}_2^\tn{P}}

\newcommand{\oper}[1]{\ensuremath\textnormal{\smaller{\textsf{#1}}}}

\newcommand{\biglor}{\ensuremath\bigvee}
\newcommand{\bigland}{\ensuremath\bigwedge}

\newcommand{\ter}{\oper{Terminal}}
\newcommand{\ele}{\oper{Element}}
\newcommand{\dec}{\oper{Decision}}
\newcommand{\lval}{\oper{Lit}}
\newcommand{\lfeat}{\oper{Feat}}

\newcommand{\parent}{\mathsf{parent}}
\newcommand{\childn}{\oper{children}}
\newcommand{\satisfy}{\oper{Sat}}

\newcounter{tableeqn}[table]
\renewcommand{\thetableeqn}{\thetable.\arabic{tableeqn}}

\DeclareMathOperator*{\lequiv}{\leftrightarrow}
\DeclareMathOperator*{\limply}{\rightarrow}

\DeclareMathOperator*{\outdeg}{\tn{deg}^{\tn{+}}}

\definecolor{gray}{rgb}{.4,.4,.4}
\definecolor{midgrey}{rgb}{0.5,0.5,0.5}
\definecolor{middarkgrey}{rgb}{0.35,0.35,0.35}
\definecolor{darkgrey}{rgb}{0.3,0.3,0.3}
\definecolor{darkred}{rgb}{0.7,0.1,0.1}
\definecolor{midblue}{rgb}{0.2,0.2,0.7}
\definecolor{darkblue}{rgb}{0.1,0.1,0.5}
\definecolor{darkgreen}{rgb}{0.1,0.5,0.1}
\definecolor{defseagreen}{cmyk}{0.69,0,0.50,0}

\newcommand{\jnoteF}[1]{}

\newcolumntype{L}[1]{>{\raggedright\let\newline\\\arraybackslash\hspace{0pt}}m{#1}}
\newcolumntype{C}[1]{>{\centering\let\newline\\\arraybackslash\hspace{0pt}}m{#1}}
\newcolumntype{R}[1]{>{\raggedleft\let\newline\\\arraybackslash\hspace{0pt}}m{#1}}

\tikzset{
  0 my edge/.style={densely dashed, my edge},
  my edge/.style={-{Stealth[]}},
}

\setlist{nosep,leftmargin=0.45cm}

\providetoggle{long}
\settoggle{long}{false}

\title{On Deciding Feature Membership in\\
  Explanations of SDD \& Related Classifiers}

\author{%
Xuanxiang Huang$^1$\and
Joao Marques-Silva$^2$ \\
\affiliations
$^1$University of Toulouse, France\\
$^2$IRIT, CNRS, Toulouse, France\\
\emails
xuanxiang.huang@univ-toulouse.fr,
joao.marques-silva@irit.fr
}

\begin{document}

\maketitle

\begin{abstract}
  When reasoning about explanations of machine Learning (ML)
  classifiers, a pertinent query is to decide whether some sensitive
  features can serve for explaining a given prediction.
  %
  %
  Recent work showed that the feature membership problem (FMP)
  is hard for $\stwop$ for a broad class of classifiers.
  In contrast, this paper shows that for a number of families of
  classifiers, 
  FMP is in NP.
  %
  Concretely, the paper proves that any classifier for which an
  explanation can be computed in polynomial time, then deciding
  feature membership in an explanation can be decided with one NP
  oracle call.
  The paper then proposes propositional encodings for classifiers
  represented with Sentential Decision Diagrams (SDDs) and for other
  related propositional languages.
  %
  %
  The experimental results confirm the practical efficiency of the
  proposed approach.
\end{abstract}

\section{Introduction} \label{sec:intro}

There is a growing interest in eXplainable Artificial Intelligence
(XAI)~\cite{pedreschi-acmcs19,zhu-nlpcc19}. 
This interest is explained in part by the ongoing advances in Machine
Learning (ML) and the resulting uses of ML in settings that impact
humans, including high-risk and safety-critical
applications~\cite{eu-aiact21}.
However, XAI finds other important uses~\cite{weld-cacm19}. XAI can
serve for diagnosing systems that exploit ML. XAI can be used to train
human operators so that they learn from ML-enabled systems.
Most importantly, XAI offers a general instrument for building trust
in the use of systems of ML.

Most of past work on XAI involves so-called model-agnostic approaches.
Model-agnostic XAI offers a practical solution for explaining complex
ML models, and has been deployed in a number of relevant
applications\footnote{E.g.~\url{https://cloud.google.com/explainable-ai}.}.
However, model-agnostic XAI offers no guarantees of rigor, and can
even (and often) produce unsound explanations~\cite{ignatiev-ijcai20}.
Thus, the use of model-agnostic XAI solutions in high-risk and
safety-critical applications is ill-advised, as the lack of rigor
could induce human decision makers in error.
Recent years have seen the inception of formal approaches to XAI
(FXAI)~\cite{darwiche-ijcai18,inms-aaai19,darwiche-ecai20}.
FXAI offers the strongest guarantees of rigor, since reasoning is in
most cases model-precise, i.e.\ the actual ML model is accounted for
when reasoning about explanations, and so explanations are rigorous
with respect to the (logic) representation of the ML model.

Besides the computation of formal explanations, FXAI can answer a
number of additional queries~\cite{marquis-kr20,hiims-kr21}.
Concretely, this paper studies the problem of deciding whether a
feature can occur in some explanation of a given prediction for an ML
classifier.
In some practical uses of an ML classifier, it may be critical to
decide whether a sensitive feature can be used in some explanation. 
For example, for a bank loan application, it would be troubling if a
feature like gender, age, or ethnic origin might serve to explain a
decision on a bank loan.
Recent work~\cite{hiims-kr21} proved that, for classifiers represented 
as DNF (disjunctive normal form) formulas, feature membership is hard
for $\stwop$. Thus, deciding feature membership should in general be
at least as hard as solving a quantified boolean formula with two
levels of quantifiers.
However, it was also shown~\cite{hiims-kr21} that FMP can be
decided in polynomial time in the case of decision trees (DTs), and
that the problem is in NP for the case of classifiers that can be
represented with explanation graphs (XpG's).

The gap in the computational complexity of FMP between DNF formulas
and DTs (and also XpG's) suggests that, for classifiers represented
with specific propositional languages, the complexity of FMP could be
simpler than that of DNF formulas.
This paper proves that this is indeed the case.
The paper starts by proving a more general result, namely that for any
classifier for which one explanation can be computed in polynomial
time, then FMP is in NP (and so FMP can be decided with an oracle for
NP).
The proof of this result offers a general approach for solving FMP,
which entails devising propositional encodings for the target
classifiers. However, the general approach can require large
propositional encodings, which Boolean satisfiability (SAT) reasoners
may be unable to solve efficiently.
As a result, the paper refines the general result, proposing an
alternative simpler approach for deciding FMP.
As demonstrated by the experiments, the proposed refined approach
yields much more compact encodings, which in turn enables SAT solvers
to efficiently decide FMP for different families of classifiers.
Furthermore, the paper details how the proposed approach can be
instantiated for two concrete families of classifiers, namely those
represented with Sentential Decision Diagrams
(SDDs)~\cite{darwiche-ijcai11}, but also those for which the problem
of computing one explanation can be represented with an explanation
graph (XpG's)~\cite{hiims-kr21}.
The experimental results confirm that FMP can be decided for large
SDDs (and also large XpG's) for classification problems with a large
number of features.

\section{Preliminaries} \label{secLprelim}


\paragraph{Classification problems.}
This paper considers classification problems, which are defined on a
set of features (or attributes) $\fml{F}=\{1,\ldots,m\}$ and a set of
classes $\fml{K}=\{c_1,c_2,\ldots,c_K\}$.
Each feature $i\in\fml{F}$ takes values from a domain $\mbb{D}_i$.
In this paper, and unless otherwise indicate, the domains and classes
are assumed to be boolean, i.e.\ $\mbb{B}=\{0,1\}$ and
$\fml{K}=\{0,1\}$. (It will also be convenient to allow
$\fml{K}=\{\bot,\top\}$ for propositional languages.)
%
%
Feature space is defined as
$\mbb{F}=\mbb{D}_1\times{\mbb{D}_2}\times\ldots\times{\mbb{D}_m}=\mbb{B}^m$;
%
%
The notation $\mbf{x}=(x_1,\ldots,x_m)$ denotes an arbitrary point in
feature space, where each $x_i$ is a variable taking values from
$\mbb{D}_i$. The set of variables associated with features is
$X=\{x_1,\ldots,x_m\}$.
Moreover, the notation $\mbf{v}=(v_1,\ldots,v_m)$ represents a
specific point in feature space, where each $v_i$ is a constant
representing one concrete value from $\mbb{D}_i$. 
An ML classifier $\mbb{C}$ is characterized by a (non-constant)
\emph{classification function} $\kappa$ that maps feature space
$\mbb{F}$ into the set of classes $\fml{K}$,
i.e.\ $\kappa:\mbb{F}\to\fml{K}$.
An \emph{instance} 
denotes a pair $(\mbf{v}, c)$, where $\mbf{v}\in\mbb{F}$ and
$c\in\fml{K}$, with $c=\kappa(\mbf{v})$. 
%

\paragraph{Formal explanations.}
In contrast with well-known model-agnostic explanation
approaches~\cite{guestrin-kdd16,lundberg-nips17,guestrin-aaai18,pedreschi-acmcs19}, 
formal explanations are rigorously defined in terms of the function 
computed by the classifier.
Prime implicant (PI) explanations~\cite{darwiche-ijcai18} denote a
minimal set of literals (relating a feature value $x_i$ and a constant
$v_i\in\mbb{D}_i$) 
that are sufficient for the prediction. PI-explanations are related
with abduction, and so are also referred to as abductive explanations
($\axp$)~\cite{inms-aaai19}.
More recently, PI-explanations have been studied
in terms of their computational
complexity~\cite{barcelo-nips20,marquis-kr21}.
Additional examples of recent work on formal explanation
includes~\cite{kutyniok-jair21,kwiatkowska-ijcai21,mazure-cikm21,tan-nips21}.

Formally, given $\mbf{v}=(v_1,\ldots,v_m)\in\mbb{F}$, with
$\kappa(\mbf{v})=c$, an $\axp$ is any minimal subset
$\fml{X}\subseteq\fml{F}$ such that, 
\begin{equation} \label{eq:axp}
  \forall(\mbf{x}\in\mbb{F}).
  \left[
    \bigland\nolimits_{i\in{\fml{X}}}(x_i=v_i)
    \right]
  \limply(\kappa(\mbf{x})=c)
\end{equation}
i.e.\ if the features in $\fml{X}$ is sufficient for the predictions
when these take the values dictated by $\mbf{v}$, and $\fml{X}$ is
irreducible.
%
$\axp$'s can be viewed as answering a `Why?' question, i.e.\ why is some 
prediction made given some point in feature space.
Besides, any subset $\fml{X}'\subseteq\fml{F}$ satisfying
\eqref{eq:axp} is called a \emph{weak} $\axp$ ($\waxp$).
In other words, an $\axp$ is a \emph{subset-minimal} or irreducible
$\waxp$.
Given a set $\fml{X}\subseteq\fml{F}$, the predicate $\waxp(\fml{X})$
is true iff $\fml{X}$ is a weak $\axp$~\footnote{With a mild abuse of
  notation we use the symbols $\waxp$ and $\axp$ to also denote
  predicates defined on sets of features, representing the condition
  of a set denoting, respectively, a $\waxp$ or a $\axp$. We will
  apply the same rationale for other definitions.}.
Similarly, $\axp(\fml{X})$ is true iff $\fml{X}$ is a subset-minimal
$\waxp$.
A different view of explanations is a contrastive
explanation~\cite{miller-aij19}, which answers a `Why Not?' question,
i.e.\ which features can be changed to change the prediction.
A formal definition of contrastive explanation 
is proposed in recent work~\cite{inams-aiia20}.
Given $\mbf{v}=(v_1,\ldots,v_m)\in\mbb{F}$ with $\kappa(\mbf{v})=c$, a
contrastive explanation ($\cxp$) is any minimal set
$\fml{Y}\subseteq\fml{F}$ such that,
\begin{equation} \label{eq:cxp}
  \exists(\mbf{x}\in\mbb{F}).\bigland\nolimits_{j\in\fml{F}\setminus\fml{Y}}(x_j=v_j)\land(\kappa(\mbf{x})\not=c) 
\end{equation}
Likewise, any $\fml{Y}'\subseteq\fml{F}$ satisfying \eqref{eq:cxp}
is called \emph{weak} $\cxp$ ($\wcxp$).
Given a set $\fml{Y}\subseteq\fml{F}$, the predicate $\wcxp(\fml{Y})$
is true iff $\fml{Y}$ is a weak $\cxp$.
Similarly, $\cxp(\fml{Y})$ is true iff $\fml{Y}$ is a subset-minimal
$\wcxp$.
A consequence of the definition of $\waxp(\fml{X})$ and
$\wcxp(\fml{Y})$ is that these predicates are monotone:

\begin{proposition} \label{prop:xpmono}
  If $\waxp(\fml{X})$ (resp.~$\wcxp(\fml{Y})$) holds for
  $\fml{X}\subseteq\fml{F}$ (resp.~$\fml{Y}\subseteq\fml{F}$),
  then
  $\waxp(\fml{X}')$ (resp.~$\wcxp(\fml{Y}')$) also holds for any
  $\fml{X}\subseteq\fml{X}'\subseteq\fml{F}$
  (resp.~$\fml{Y}\subseteq\fml{Y}'\subseteq\fml{F}$).
\end{proposition}

Building on the results of R.~Reiter in model-based
diagnosis~\cite{reiter-aij87},~\cite{inams-aiia20} proves a minimal
hitting set (MHS) duality
relation between $\axp$s and $\cxp$s, i.e.\ $\axp$s are MHSes of $\cxp$s and
vice-versa.
%
%

\paragraph{SDD classifiers.}
SDDs represent a well-known
propositional language~\cite{darwiche-ijcai11,darwiche-aaai15}
that support efficient operations for building and manipulating
Boolean functions.
Similar to other circuit-based representations, e.g.\ binary decision
diagrams (BDDs) or decision graphs~\cite{hiims-kr21}, SDDs can be used
as binary classifiers~\cite{hiicams-aaai22,hiicams-corr21}.
SDDs are based on a decomposition type~\cite{darwiche-ijcai11}
called \textit{partitions} which can decompose a
Boolean function as $(p_1 \land q_1) \lor \dots \lor (p_n \land q_n)$,
where each $p_i$ is called a \textit{prime} and each
$q_i$ is called a \textit{sub}.
The \textit{primes} are mutually exclusive, exhaustive and non-false.
%
What's more, the process of decomposition
is governed by a variable tree (\textit{vtree})
\cite{darwiche-ijcai11}.

As depicted in \autoref{fig:sdd_representation}, an SDD is a directed
acyclic graph (DAG) defined on $\mbb{B}^{m}$. 
Each circled node with outgoing edges is a \textit{decision node} and
represents the disjunction of its children.
Each paired-box node is an \textit{element} and represents the
conjunction of the two boxes.
The left (resp.~right) box represents the prime (resp.~sub).
A box either contains a \textit{terminal} SDD (i.e. $\top$, $\bot$ or a literal) or a link to a decision node.
\autoref{fig:sdd_vtree} shows a \textit{balanced} vtree,
where each leaf is a feature/variable.

\begin{figure}[]
	\centering
	\begin{subfigure}[b]{0.3\textwidth}
		\centering
		\scalebox{0.7}{\begin{tikzpicture}[>=latex',line join=bevel]
		node distance={1.5cm},
		\node (0) [label=above:{\small $n^k_1$}] at (0, 0) [draw,fill=white,circle] {6};
		\node (1) [label=left:{\small $n^k_2$}] at (-2, -1.5) [draw,fill=white,rectangle split, rectangle split horizontal,
		rectangle split parts=2] {\nodepart{one}$$\nodepart{two}$\top$};
		\node (2) [label=left:{\small $n^k_3$}] at (0, -1.5) [draw,fill=white,rectangle split, rectangle split horizontal,
		rectangle split parts=2] {\nodepart{one}$$\nodepart{two}$W$};
		\node (3) [label=left:{\small $n^k_4$}] at (2, -1.5) [draw,fill=white,rectangle split, rectangle split horizontal,
		rectangle split parts=2] {\nodepart{one}$\neg P$\nodepart{two}$$};
		\node (4) [label=left:{\small $n^k_5$}] at (-2, -3) [draw,fill=white,circle] {2};
		\node (5) [label=below:{\small $n^k_8$}] at (-3.25, -4.5) [draw,fill=white,rectangle split, rectangle split horizontal,
		rectangle split parts=2] {\nodepart{one}$P$\nodepart{two}$Y$};
		\node (6) [label=below:{\small $n^k_9$}] at (-1.5, -4.5) [draw,fill=white,rectangle split, rectangle split horizontal,
		rectangle split parts=2] {\nodepart{one}$\neg P$\nodepart{two}$\bot$};
		\node (7) [label=left:{\small $n^k_6$}] at (0, -3) [draw,fill=white,circle] {2};
		\node (8) [label=left:{\small $n^k_7$}] at (2, -3) [draw,fill=white,circle] {5};
		\node (9) [label=below:{\small $n^k_{10}$}] at (0, -4.5) [draw,fill=white,rectangle split, rectangle split horizontal,
		rectangle split parts=2] {\nodepart{one}$P$\nodepart{two}$\neg Y$};
		\node (10) [label=below:{\small $n^k_{11}$}] at (1.5, -4.5) [draw,fill=white,rectangle split, rectangle split horizontal,
		rectangle split parts=2] {\nodepart{one}$M$\nodepart{two}$W$};
		\node (11) [label=below:{\small $n^k_{12}$}] at (3, -4.5) [draw,fill=white,rectangle split, rectangle split horizontal,
		rectangle split parts=2] {\nodepart{one}$\neg M$\nodepart{two}$\bot$};
		\draw [->] (0) -- (1);
		\draw [->] (0) -- (2);
		\draw [->] (0) -- (3);
		\draw [*->] (1.one)+(0.2em,0.4em) -- (4);
		\draw [->] (4) -- (5);
		\draw [->] (4) -- (6);
		\draw [*->] (2.one)+(0.2em,0.4em) -- (7);
		\draw [->] (7) -- (6);
		\draw [->] (7) -- (9);
		\draw [*->] (3.two)+(0.2em,0.4em) -- (8);
		\draw [->] (8) -- (10);
		\draw [->] (8) -- (11);
\end{tikzpicture}}
		\caption{An SDD classifier}
		\label{fig:sdd_representation}
	\end{subfigure}
	\begin{subfigure}[b]{0.1\textwidth}
		\centering
		\scalebox{0.6}{\begin{tikzpicture}[>=latex',line join=bevel,]
\node (n3) at (63.0bp,109.5bp) [draw,draw=none] {6};
  \node (n1) at (45.0bp,60.5bp) [draw,draw=none] {2};
  \node (n5) at (81.0bp,60.5bp) [draw,draw=none] {5};
  \node (n0) at (9.0bp,9.0bp) [draw,draw=none] {P};
  \node (n2) at (45.0bp,9.0bp) [draw,draw=none] {Y};
  \node (n4) at (81.0bp,9.0bp) [draw,draw=none] {M};
  \node (n6) at (117.0bp,9.0bp) [draw,draw=none] {W};
  \draw [] (n3) ..controls (57.301bp,93.618bp) and (50.546bp,75.981bp)  .. (n1);
  \draw [] (n3) ..controls (68.699bp,93.618bp) and (75.454bp,75.981bp)  .. (n5);
  \draw [] (n1) ..controls (34.517bp,45.086bp) and (22.314bp,28.307bp)  .. (n0);
  \definecolor{strokecol}{rgb}{0.0,0.0,0.0};
  \pgfsetstrokecolor{strokecol}
  \draw (11.99bp,23.736bp) node {0};
  \draw [] (n1) ..controls (45.0bp,45.086bp) and (45.0bp,28.307bp)  .. (n2);
  \draw (41.0bp,23.736bp) node {1};
  \draw [] (n5) ..controls (81.0bp,45.086bp) and (81.0bp,28.307bp)  .. (n4);
  \draw (85.0bp,23.736bp) node {3};
  \draw [] (n5) ..controls (91.483bp,45.086bp) and (103.69bp,28.307bp)  .. (n6);
  \draw (112.91bp,23.736bp) node {4};
\end{tikzpicture}}
		\caption{A vtree}
		\label{fig:sdd_vtree}
	\end{subfigure}
	\caption{SDD representation for a binary classification function $\kappa(P,Y,M,W)=(Y \land P) \lor (P \land W) \lor (W \land M)$,
	given a vtree.}
	\label{fig:sdd_example}
\end{figure}
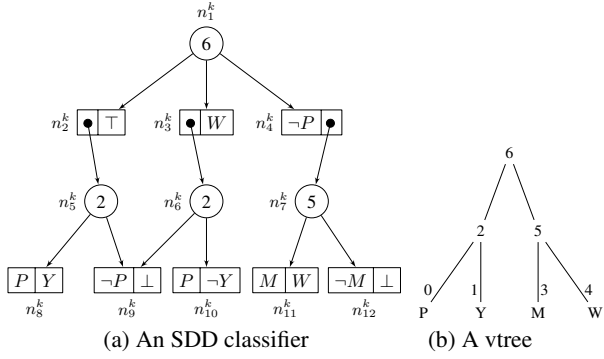

\paragraph{Queries and transformations.}
In this paper, we only consider a number of queries and transformations
that are supported by SDDs; these  are the query \tbf{CO} (polytime
consistency check), and the transformations \tbf{CD} (polytime
conditioning)~\footnote{Note that canonical SDDs don't support
  polytime conditioning~\cite{darwiche-aaai15}.}
and $\neg$\tbf{C} (polytime negation).
%
%
Let \tbf{L} denotes a propositional language and $\rho$ denotes a term (i.e. conjunction of literals),
we have the following standard definitions.
\begin{definition}[Conditioning~\cite{darwiche-jair02}]
	Let $\Phi$ represent a propositional formula and let $\rho$ denote
	a consistent term. The \emph{conditioning} of $\Phi$ on $\rho$,
	i.e.\ $\Phi|_{\rho}$, is the formula obtained by replacing each
	variable $x_i$ by $\top$ (resp.~$\bot$) if $x_i$ (resp.~$\neg{x_i}$)
	is a positive (resp.~negative) literal of $\rho$.
\end{definition}
\begin{definition}[Queries \& transformations~\cite{darwiche-jair02}]
	The following queries and transformations are used throughout with
	respect to a propositional language \tbf{L}:
	\begin{itemize}[nosep]
		\item \tbf{L} satisfies \tbf{CO}
		iff there exists a polytime algorithm that maps
		every formula $\Phi$ from \tbf{L} to 1 if $\Phi$ is consistent,
		and to 0 otherwise.
		\item \tbf{L} satisfies \tbf{CD} iff
		there exists a polytime algorithm that maps every formula
		$\Phi$ from \tbf{L} and every consistent term $\rho$ into a
		formula from \tbf{L} that is logically equivalent to $\Phi|_{\rho}$.
		\item \tbf{L} satisfies $\neg$\tbf{C} iff there exists a
		polytime algorithm that maps every formula $\Phi$ from \tbf{L}
		to a formula of \tbf{L} that is logically equivalent to $\neg \Phi$.
	\end{itemize}
\end{definition}
%

\paragraph{Related classifiers \& XpGs.}
Apart from SDDs, we also consider other graph-based classifiers,
for which the computation of one explanation can be represented with
explanation graphs (XpG's)~\cite{hiims-kr21} (and references therein).
Concrete examples include Decision Trees
(DTs)~\cite{quinlan1986induction}, Ordered Binary Decision Diagrams
(OBDDs)~\cite{bryant-tcomp86}, Ordered Multi-Valued Decision Diagrams
(OMDDs)~\cite{brayton-tr90} and Decision Graphs
(DGs)~\cite{oliver-tr92}. (For DTs, DGs and OMDDs both the domains of
features and the set of classes may not be boolean.)
We include below a brief overview XpG's~\cite{hiims-kr21}.
\begin{definition}[Explanation Graph (XpG)]
An XpG is a 5-tuple
$\fml{D}=(G_{\fml{D}},S,\upsilon,\alpha_{V},\alpha_{E})$, where:
    \begin{enumerate}[nosep]
    \item $G_{\fml{D}}=(V_{\fml{D}},E_{\fml{D}})$ is a labeled DAG, such
    that:
    \begin{itemize}[nosep]
    \item $V_{\fml{D}}=T_{\fml{D}}\cup{N_{\fml{D}}}$ is the set of
    nodes, partitioned into the terminal nodes $T_{\fml{D}}$
    (with $\outdeg(q)=0$, $q\in{T_{\fml{D}}}$) and the non-terminal
    nodes $N_{\fml{D}}$ (with $\outdeg(p)>0$, $p\in{N_{\fml{D}}}$);
    \item $E_{\fml{D}}\subseteq{V_{\fml{D}}}\times{V_{\fml{D}}}$ is
    the set of (directed) edges.
    \item $G_{\fml{D}}$ is such that there is a single node with
    indegree equal to 0, i.e.\ the root (or source) node.
    \end{itemize}
    \item $S=\{s_1,\ldots,s_m\}$ is a set of variables;
    \item $\upsilon:N_{{\fml{D}}}\to{S}$ is a total function mapping
    each non-terminal node to one variable in $S$.
    \item $\alpha_V:V_{\fml{D}}\to\{0,1\}$ labels nodes with one of two
    values.\\
    ($\alpha_{V}$ is required to be defined only for terminal nodes.)
    \item $\alpha_E:E_{\fml{D}}\to\{0,1\}$ labels edges with one of two
    values.
    \end{enumerate}
    In addition, an XpG $\fml{D}$ must respect the following properties:
    \begin{enumerate}[nosep,label=\roman*.]
    \item For each non-terminal node, there is at most one outgoing edge
    labeled 1; all other outgoing edges are labeled 0.
    \item There is exactly one terminal node $t\in{T}$ labeled 1 that
    can be reached from the root node with (at least) one path of
    edges labeled 1.
    \end{enumerate}
\end{definition}
We refer to a \emph{tree XpG} when the DAG associated with the XpG is
a tree.
Given a DAG $\fml{G}$ representing a classifier
$\mbb{C} \in \{\text{DTs}, \text{OBDDs}, \text{OMDDs}, \text{DGs}\}$,
and an instance $(\mbf{v},c)$, the (unique)
mapping to an XpG is obtained as follows:
\begin{enumerate}[nosep]
    \item The same DAG is used.
    \item Terminal nodes labeled $c$ in $\fml{G}$ are labeled 1 in
    $\fml{D}$. Terminal nodes labeled $c'\not=c$ in $\fml{G}$ are
    labeled 0 in $\fml{D}$.
    \item A non-terminal node associated with feature $i$ in $\fml{G}$
    is associated with $s_i$ in $\fml{D}$.
    \item Any edge labeled with a literal that is consistent with
    $\mbf{v}$ in $\fml{G}$ is labeled 1 in $\fml{D}$. Any edge labeled
    with a literal that is not consistent with $\mbf{v}$ in $\fml{G}$ is
    labeled 0 in $\fml{D}$.
\end{enumerate}
%
%
%
\autoref{fig:xpg_representation} shows an XpG
mapped from an OBDD classifier (\autoref{fig:bdd_representation})
and an instance.
\subsubsection{Evaluation of XpG's.}
Given an XpG $\fml{D}$, let $\mbb{S}=\mbb{B}^{m}$, i.e.\ the set of
possible assignments to the variables in $S$. The evaluation function
of the XpG, $\sigma_{\fml{D}}:\mbb{S}\to\{0,1\}$, is based on the
auxiliary \emph{activation} function
$\varepsilon:\mbb{S}\times{V_{\fml{D}}}\to\{0,1\}$.
Moreover, for a point $\mbf{s}\in\mbb{S}$, $\sigma_{\fml{D}}$ and
$\varepsilon$ are defined as follows:
\begin{enumerate}[nosep]
\item If $j$ is the root node of $G_{\fml{D}}$, then
    $\varepsilon(\mbf{s},j)=1$.
    \item Let $p\in\parent(j)$ (i.e.\ a node can have multiple parents)
    and let $s_i=\upsilon(p)$.
    $\varepsilon(\mbf{s},j)=1$ iff $\varepsilon(\mbf{s},p)=1$ and either
    $\alpha_{E}(p,j)=1$ or $s_i=0$, i.e.
    \begin{equation}\small
    \varepsilon(\mbf{s},j)\lequiv
    \biglor_{\substack{p\in\parent(j)\\\land\neg\alpha_{E}(p,j)}}\left(\varepsilon(\mbf{s},p)\land\neg{s_i}\right)
    \biglor_{\substack{p\in\parent(j)\\\land\alpha_{E}(p,j)}}\varepsilon(\mbf{s},p)
    \end{equation}
    \item $\sigma_{\fml{D}}(\mbf{s})=1$ iff for every terminal node
    $j\in{T_{\fml{D}}}$, with $\alpha_{V}(j)=0$, it is also the case that
    $\varepsilon(\mbf{s},j)=0$, i.e.
    \begin{equation}\small
    \sigma_{\fml{D}}(\mbf{s})\lequiv
    \bigland\nolimits_{j\in{T_{\fml{D}}}\land\neg\alpha_{V}(j)}\neg\varepsilon(\mbf{s},j)
    \end{equation}
\end{enumerate}
Terminal nodes labeled 1 are irrelevant for defining $\sigma_{\fml{D}}$.
Their existence is implicit (i.e.\ at least one
terminal node with label 1 must exist and be reachable from the root
when all the $s_i$ variables take value 1), but the evaluation of
$\sigma_{\fml{D}}$ is oblivious to their existence.
Furthermore, and as noted above, we must have
$\sigma_{\fml{D}}(1,\ldots,1)=1$.
If the graph has some terminal node labeled 0, then
$\sigma_{\fml{D}}(0,\ldots,0)=0$.
This implies that
$\sigma_{\fml{D}} = 1$ if the prediction of the original classifier remain unchanged,
and $\sigma_{\fml{D}} = 0$ if the prediction of the original classifier changed.
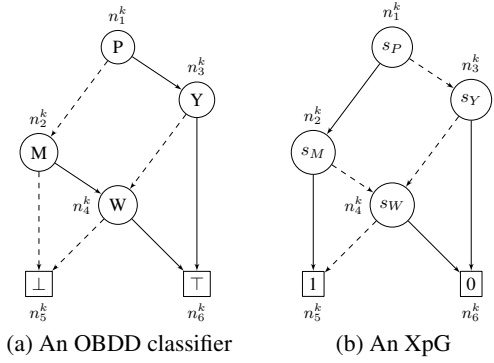
\begin{figure}[]
	\centering
	\begin{subfigure}[b]{0.2\textwidth}
	\centering
		\scalebox{0.7}{\begin{tikzpicture}[>=latex',line join=bevel]
		node distance={1.5cm},
		\node (0) [label=above:{\small $n^k_1$}] at (0, 0) [draw,fill=white,circle] {P};
		\node (1) [label=above:{\small $n^k_2$}] at (-1.5, -2) [draw,fill=white,circle] {M};
		\node (2) [label=above:{\small $n^k_3$}] at (1.5, -1) [draw,fill=white,circle] {Y};
		\node (3) [label=left:{\small $n^k_4$}] at (0, -3) [draw,fill=white,circle] {W};
		\node (4) [label=below:{\small $n^k_5$}] at (-1.5, -4.5) [draw,fill=white,rectangle] {$\bot$};
		\node (5) [label=below:{\small $n^k_{6}$}] at (1.5, -4.5) [draw,fill=white,rectangle] {$\top$};
		\draw [->] (0) edge [dashed] (1);
		\draw [->] (0) -- (2);
		\draw [->] (1) -- (3);
		\draw [->] (1) edge [dashed] (4);
		\draw [->] (2) edge [dashed] (3);
		\draw [->] (2) -- (5);
		\draw [->] (3) edge [dashed] (4);
		\draw [->] (3) -- (5);
\end{tikzpicture}}
		\caption{An OBDD classifier}
		\label{fig:bdd_representation}
	\end{subfigure}
	\begin{subfigure}[b]{0.2\textwidth}
		\centering
		\scalebox{0.7}{\begin{tikzpicture}[>=latex',line join=bevel]
		node distance={1.5cm},
		\node (0) [label=above:{\small $n^k_1$}] at (0, 0) [draw,fill=white,circle] {$s_P$};
		\node (1) [label=above:{\small $n^k_2$}] at (-1.5, -2) [draw,fill=white,circle] {$s_M$};
		\node (2) [label=above:{\small $n^k_3$}] at (1.5, -1) [draw,fill=white,circle] {$s_Y$};
		\node (3) [label=left:{\small $n^k_4$}] at (0, -3) [draw,fill=white,circle] {$s_W$};
		\node (4) [label=below:{\small $n^k_5$}] at (-1.5, -4.5) [draw,fill=white,rectangle] {1};
		\node (5) [label=below:{\small $n^k_{6}$}] at (1.5, -4.5) [draw,fill=white,rectangle] {0};
		\draw [->] (0) -- (1);
		\draw [->] (0) edge [dashed] (2);
		\draw [->] (1) edge [dashed] (3);
		\draw [->] (1) -- (4);
		\draw [->] (2) edge [dashed] (3);
		\draw [->] (2) -- (5);
		\draw [->] (3) edge [dashed] (4);
		\draw [->] (3) -- (5);
\end{tikzpicture}}
		\caption{An XpG}
		\label{fig:xpg_representation}
	\end{subfigure}
	\caption{OBDD representation for $\kappa(P,Y,M,W)=(Y \land P) \lor (P \land W) \lor (W \land M)$,
	and its corresponding XpG given the instance Ella = \{Y, $\neg$ P, W, $\neg$ M\}.
	Edges corresponding to value 0 (resp. value 1) are indicated by dashed lines (resp. solid lines).
	Non-terminals are represented as circle nodes, terminal nodes are represented as boxes.}
	\label{fig:xpg_example}
\end{figure}

\paragraph{Feature membership.}
Let $\mbb{C}$ be a classifier defined on a set of features $\fml{F}$,
a set of classes $\fml{K}$, with feature space $\mbb{F}$, and
computing function $\kappa$.
The feature membership considered in this paper is adapted from
earlier work~\cite{hiims-kr21}:
\begin{definition}
  Given a classifier $\mbb{C}$, an instance $(\mbf{v},c)$ and a
  feature $r\in\fml{F}$, the feature membership problem (FMP) is to
  decide whether target feature $t$ is included in some explanation of 
  instance $(\mbf{v},c)$.
\end{definition}
Previous work~\cite{hiims-kr21} established that for a DNF classifier,
FMP is $\stwop$-hard, but that for DTs, FMP is in P.
Moreover,~\cite{hiims-kr21} proved that a target feature $t$ is
included in some of the $\axp$s iff it is included in some of the
$\cxp$s.
As a result, in this paper, we will focus mainly of deciding FMP on
some $\axp$s.
One additional result in~\cite{hiims-kr21} is a proof that FMP for
XpG's is in NP.
\begin{example}
	Throughout the paper, we consider a staff recruitment scenario
	as our running example.
	In this scenario, we have a binary classification function
	$\kappa(P,Y,M,W)=(Y \land P) \lor (P \land W) \lor (W \land M)$.
	Its input are four features:
	1) \textbf{Y}oung is true if the age of an applicant is less than 24; 
	2) To\textbf{P} is true if the applicant graduated from a top university.
	3) \textbf{M}ale is true if the applicant is male.
	4) \textbf{W}ork is true if the applicant has work experience.
	Its output is either $\top$ (accept) or $\bot$ (reject).
	Applicant Ella = \{Y, $\neg$ P, W, $\neg$ M\} get $\bot$.
	\autoref{fig:sdd_representation} shows the SDD representation
	of this classification function. \autoref{fig:bdd_representation}
	shows the OBDD representation of this classification function,
	and \autoref{fig:xpg_representation} shows the XpG representation
	of this OBDD and this applicant Ella.
	To test if this classifier is biased on feature \textbf{M}ale,
	we solve the query:
	\emph{is there an $\axp$ containing feature \textbf{M}ale}.
\end{example}

\section{Classifiers with FMP in NP}

This section proves results that are used throughout. First, we
prove that finding an $\axp$/$\cxp$ of an SDD classifier runs in
polynomial time. Second, we prove that for any classifier for which
computing one $\axp$/$\cxp$ runs in polynomial time, then deciding FMP is in NP.

\subsection{Finding one AXp and CXp for SDD Classifiers}
We assume that the target binary classification functions are
completely specified. This means that for any point in feature space,
the classifier either predicts $\top$ or $\bot$. 
\begin{proposition} \label{prop:sddaxp}
  Finding one AXp of a decision taken by a SDD $\mbb{C}$ is
  polynomial-time.
\end{proposition}
\begin{proof}
  Let $(\mbf{v},c)$ be such that $c = \bot$.
  Our goal is then to find a $\fml{X}$ such that $\kappa|_{\left[\bigland\nolimits_{i\in{\fml{X}}}(x_i=v_i)\right]}$
  is inconsistent with the features in $\fml{X}$ fixed, but becomes
  consistent if any feature $i$ is removed from $\fml{X}$. 
  Since SDD satisfies \tbf{CD} and \tbf{CO}, then fixing feature
  $i\in\fml{X}$ to the $v_i$ (i.e.\ coordinate $i$ of $\mbf{v}$) 
  can be done in polynomial time, and checking the consistency of
  the $\kappa|_{\left[\bigland\nolimits_{i\in{\fml{X}}}(x_i=v_i)\right]}$
  can also be done in polynomial time.\\
  In the case of $c = \top$. Since SDD satisfies $\neg$\tbf{C},
  then we can construct a new SDD classifier $\mbb{C}'$ in polynomial
  time by using the negation operation. Then any instance classified
  as $\top$ in the original classifier $\mbb{C}$ is classified as
  class $\bot$ in the new classifier $\mbb{C}'$.
  This means finding an $\axp$ $\fml{X}$ of an instance with prediction
  $\top$ in the original classifier $\mbb{C}$ can be done in
  polynomial time in the new classifier $\mbb{C}'$.
\end{proof}
\begin{proposition} \label{prop:sddcxp}
  Finding one CXp of a decision taken by a SDD $\mbb{C}$ is
  polynomial-time.
\end{proposition}
\cref{prop:sddcxp} can be proved with the similar argument described in the proof of \cref{prop:sddaxp}.
But the difference is to find a $\fml{Y}$ such that $\kappa|_{\left[\bigland\nolimits_{i\in\fml{F}\setminus\fml{Y}}(x_i=v_i)\right]}$
is consistent with the features in $\fml{F} \setminus \fml{Y}$ fixed,
but becomes inconsistent if any feature $i \in \fml{Y}$ is added to $\fml{F} \setminus \fml{Y}$.
%
 %


\subsection{Classifiers with Polynomial-Time Explanations}

This section proves that, for several families of classifiers, FMP is
in NP, and so can be decided with an NP oracle call. (In contrast with
earlier work~\cite{hiims-kr21}, that includes a similar proof for
XpG's, our proof is independent of a concrete classifier, depending
only on the fact that one explanation is computed in polynomial
time.)
Concretely, we prove that, if given $\fml{X}\subseteq\fml{F}$,
deciding~\eqref{eq:axp} (or~\eqref{eq:cxp}) is in P, then deciding FMP
is in NP.

\begin{proposition} \label{prop:fmpnp}
  Given a classifier for which~\eqref{eq:axp} (or~\eqref{eq:cxp}) can
  be decided in polynomial time, then FMP is in NP.
\end{proposition}

\begin{proof}
  We reason in terms of~\eqref{eq:axp}, but a similar argument could
  be used in the case of ~\eqref{eq:cxp}.\\
  To prove that a set $\fml{X}$ is an $\axp$, it suffices to prove
  that:
  \begin{enumerate}[start=1,label*={\arabic*.},nosep]
    \item $\waxp(\fml{X}) = \top$;
    \item $\forall i \in \fml{X}. \waxp(\fml{X} \setminus \{i\}) = \bot$, that is, $\fml{X}$ is subset-minimal.
  \end{enumerate}

\jnoteF{INVOKE MONOTONICITY}

%
  Now, since by hypothesis, we can decide~\eqref{eq:axp} in
  polynomial time, 
  then we can decide whether any guessed set
  $\fml{X}$ containing feature $t$ is an $\axp$ in polynomial-time, as
  follows.
  For step 1., check that $\fml{X}$ is a $\waxp$.
  For step 2., iteratively check,
  for each feature $i \in \fml{X}$, $\fml{X}\setminus\{i\}$ is not a $\waxp$.
  Clearly, given $\fml{X}$, this procedure runs in polynomial time.
  Thus FMP is in NP.
\end{proof}

Given~\cref{prop:fmpnp} (which offers an alternative proof to the
result in~\cite{hiims-kr21} for XpG's), we need now to devise ways to
exploit NP oracles for solving FMP. This is the topic of the next
sections.
%
%

\subsection{Deciding Membership Without Witnesses} \label{ssec:fmlnw}

As argued in the previous section, the proof of \cref{prop:fmpnp}
offers a solution for solving FMP in the case computing AXp's or CXp's
is in P.
As shown later, for classifiers for which there exists a propositional
encoding for deciding whether a set of features is a $\waxp$, one can
use~\cref{prop:fmpnp} to devise a propositional encoding for deciding
FMP.
However, a straightforward encoding of the approach outlined
in~\cref{prop:fmpnp} often requires large propositional formulas.
These formulas must encode one copy of the classifier to decide
whether a pick $\fml{X}$ of the features is a $\waxp$, and then $m$
copies (one for each feature) of the classifier to decide whether
$\fml{X}$ is indeed subset-minimal.
Observe that, since the size of $\fml{X}$ must be guessed, one must be
prepared to check $m$ features in the worst-case, and so the encoding
must indeed account for $m+1$ copies of the classifier

In this section, we propose an approach that leads to drastically
tighter encodings, premised on a simplification to the conditions
proposed in the proof of~\cref{prop:fmpnp}. (The conditions
of~\cref{prop:fmpnp} were also considered in earlier
work~\cite{hiims-kr21} for a concrete family of classifiers.)
Furthermore, one apparent downside of this alternative approach is
that the picked set of features $\fml{X}$ may not represent a witness
$\axp$. However, we also show how a witness $\axp$ can still be computed
from $\fml{X}$ in polynomial time. 

%

The approach proposed in this section hinges on the following result: 
\begin{proposition} \label{prop:fmpnp2}
  Let $\fml{X}\subseteq\fml{F}$ represent a pick of the features, such
  that, $\waxp(\fml{X})$ holds and $\waxp(\fml{X}\setminus\{t\})$ does
  not hold.
  Then, for any AXp $\fml{Z}\subseteq\fml{X}\subseteq\fml{F}$, it must
  be the case that $t\in\fml{Z}$.
\end{proposition}

\begin{proof}
  Let $\fml{Z}\subseteq\fml{F}$ by any AXp such that
  $\fml{Z}\subseteq\fml{X}$. Clearly, by definition $\waxp(\fml{Z})$
  must hold.
  Moreover, from \cref{prop:xpmono}, it is also the case that
  $\waxp(\fml{Z}')$ must hold, with
  $\fml{Z}'=\fml{Z}\cup(\fml{X}\setminus(\fml{Z}\cup\{t\}))$, since
  $\fml{Z}\subseteq\fml{Z}'\subseteq\fml{F}$. However, by hypothesis,
  $\waxp(\fml{X}\setminus\{t\})$ does not hold; a contradiction.
\end{proof}

\jnoteF{COROLLARY?}

When compared with~\cref{prop:fmpnp}, \cref{prop:fmpnp2} offers a
simpler test to decide whether $t$ is included in $\axp$, in that it
suffices to guess a set $\fml{X}$ which is a $\waxp$, and such that
removing $t$ will cause $\fml{X}\setminus\{t\}$ not to be a $\waxp$.
An apparent drawback of this simpler test to decide $\axp$ membership is
that the guessed set $\fml{X}$ need \emph{not} represent an $\axp$. 

Nevertheless, we can use \cref{prop:fmpnp2} to devise an efficient
algorithm for producing a witness of $t$ being included in some $\axp$.
Let $\fml{X}\subseteq\fml{F}$ be some guessed set which satisfies the
conditions of \cref{prop:fmpnp2}. Because the working assumption is
that the classifier is such that an $\axp$ can be computed in
polynomial-time, and since \emph{any} $\axp$ contained in $\fml{X}$ must
include $t$, then we can simply extract \emph{any} $\axp$ (in polynomial
time) starting from set $\fml{X}$ (which can be viewed as a seed in
algorithms proposed in earlier
work~\cite{hiims-kr21,hiicams-aaai22}.).

Since the witness $\axp$ is computed in a second step, this approach is
referred to as the \emph{two-step method}, in contrast with the
approach detailed in the proof of~\cref{prop:fmpnp}, which we refer to
as the \emph{one-step method}.
As shown in ~\cref{sec:res}, very significant performance gains can be
obtained by using the two-step method.

\section{SAT encodings of FMP for SDDs and XpGs}

This section proposes solutions for deciding FMP in the case of SDDs
and also in the case of XpG's.
The proposed propositional encoding follows the approach described in
the proofs of~\cref{prop:fmpnp} and~\cref{prop:fmpnp2}.
\paragraph{One-step method.}
This approach is based on the proof of~\cref{prop:fmpnp}.
The whole problem is encoded into $m+1$ replicas (where $m = |\fml{F}|$),
such that the 0-th replica asserts that there is a $\waxp$ $\fml{X}'$,
and each $k$-th replica asserts that if feature $k$ is included in the candidate $\fml{X}'$,
then $k$ cannot be removed from $\fml{X}'$.
Apparently, as $m+1$ replicas are required, 
this encoding is polynomial on the
number of features and the size of the classifier's representation.
What's more, it can be expected that
for SDD/XpG with a large number of features and/or number of nodes,
the size of resulting propositional encoding can be unmanageable, reaching the
limits of the scalability of SAT solvers.
\paragraph{Two-step method.}
%
%
In this approach, we seek to identify a set of features
$\fml{X}'$ that is a $\waxp$ and that contains the target feature
$t$.
More importantly, and given~\cref{prop:fmpnp2}, it is also the case
that such a set $\fml{X}'$ ensures that $t$ \emph{must} be
included in \emph{any} $\axp$ that is contained in $\fml{X}'$.
Clearly, this can be achieved with only 0-th replica and $t$-th replica.
The encoding is polynomial on the size of classifier's representation,
and in practice it scales better than the \osm.
After deciding whether there exists a $\waxp$ $\fml{X}'$
containing $t$, we can use any existing algorithm~\cite{hiims-kr21,hiicams-corr21}
for extracting one $\axp$ starting from $\fml{X}'$.

\begin{table*}[th]
	\caption{Encoding for SDD to decide whether there exists an $\axp$
	that includes feature $t$} \label{tab:enc_sdd}
	\begin{center}
	\scalebox{1.0}{
	\renewcommand{\arraystretch}{1.275}
	\renewcommand{\tabcolsep}{0.5em}
	\begin{tabular}{|c|c|c|c|} \toprule
	General Conditions on Indeces & Specific Conditions & Constraints & Fml~\# \\
	\midrule
        \multirow{8}{*}{$0\le{k}\le{m},1\le{i}\le{m},1\le{j}\le|\mbb{C}|$}
        & $\ter(j), \lfeat(j,i), \satisfy(\lval(j),v_i)$ &
        $n^k_j$
        & \refstepcounter{tableeqn} {\small(\thetableeqn)}\label{eq101}
        \\[1.25pt]
        \cline{2-4}
        & $\ter(j),\lfeat(j,i),\neg\satisfy(\lval(j),{v_i}), i = k$ &
        $n^k_j$
        & \refstepcounter{tableeqn} {\small(\thetableeqn)}\label{eq102}
        \\[1.25pt]
        \cline{2-4}
	& $\ter(j),\lfeat(j,i),\neg\satisfy(\lval(j),{v_i}), i \neq k$ &
        $n^k_j \lequiv \neg{s_i}$
        & \refstepcounter{tableeqn} {\small(\thetableeqn)}\label{eq103}
        \\[1.25pt]
        \cline{2-4}
        & $\dec(j)$  &
        $n^k_j \lequiv \biglor_{l\in\childn(j)} n^k_l$
        & \refstepcounter{tableeqn} {\small(\thetableeqn)}\label{eq104}
        \\[1.25pt]
        \cline{2-4}
        & $\ele(j)$  &
        $n^k_j \lequiv \bigland_{l\in\childn(j)} n^k_l$
        & \refstepcounter{tableeqn} {\small(\thetableeqn)}\label{eq105}
        \\[1.25pt]
        \cline{2-4}
        & $\kappa(\mbf{v})=\bot$ &
        $\neg n^0_1$
        & \refstepcounter{tableeqn} {\small(\thetableeqn)}\label{eq106}
        \\
        \cline{2-4}
        & $\kappa(\mbf{v})=\bot$ &
        $s_i \lequiv n^i_1$
        & \refstepcounter{tableeqn} {\small(\thetableeqn)}\label{eq107}
        \\
        \cline{2-4}
        & &
        $s_t$
        & \refstepcounter{tableeqn} {\small(\thetableeqn)}\label{eq108}
        \\
	\bottomrule
	\end{tabular}
	}
	\end{center}
\end{table*}

%
\begin{table*}[th] \footnotesize
	\caption{Encoding for XpG to decide whether there exists an $\axp$
	that includes feature $t$} \label{tab:enc_xpg}
	\begin{center}
	\scalebox{1.0}{
	\renewcommand{\arraystretch}{1.275}
	\renewcommand{\tabcolsep}{0.5em}
	\begin{tabular}{|c|c|c|c|} \toprule
	General Conditions on Indeces & Specific Conditions & Constraints & Fml~\# \\
	\midrule
        \multirow{8}{*}{$0\le{k}\le{m},1\le{i}\le{m},1\le{j}\le|\fml{D}|$}
        & $\neg \ter(j), k = 0$ &
        $
        n^0_j\lequiv
	\biglor_{\substack{p\in\parent(r)\\\land\neg \alpha_{E}(p,j)}}\left(n^0_p\land\neg{s_i}\right)
	\biglor_{\substack{p\in\parent(r)\\\land \alpha_{E}(p,j)}}n^0_p
	$
        & \refstepcounter{tableeqn} {\small(\thetableeqn)}\label{eq201}
        \\[1.25pt]
        \cline{2-4}
        & $\neg \ter(j)$, $k > 0$ &
        $
        n^k_j\lequiv
	\biglor_{\substack{p\in\parent(r)\\\land\neg \alpha_{E}(p,j)\land i \neq k}}\left(n^k_p\land\neg{s_i}\right)
	\biglor_{\substack{p\in\parent(r)\\\land (\alpha_{E}(p,j)\lor i=k)}}n^k_p
	$
        & \refstepcounter{tableeqn} {\small(\thetableeqn)}\label{eq202}
        \\[1.25pt]
        \cline{2-4}
        & &
        $
        \sigma^k_{\fml{D}}\lequiv
	\bigwedge\nolimits_{j\in{T_{\fml{D}}}\land \neg \alpha_{V}(j)}\neg n^k_j
	$
        & \refstepcounter{tableeqn} {\small(\thetableeqn)}\label{eq203}
        \\[1.25pt]
        \cline{2-4}
        & &
        $n^k_1$
        & \refstepcounter{tableeqn} {\small(\thetableeqn)}\label{eq204}
        \\[1.25pt]
        \cline{2-4}
        & &
        $\sigma_{\fml{D}}^0$
        & \refstepcounter{tableeqn} {\small(\thetableeqn)}\label{eq205}
        \\
        \cline{2-4}
	& &
        $s_i \lequiv \neg \sigma_{\fml{D}}^i$
        & \refstepcounter{tableeqn} {\small(\thetableeqn)}\label{eq206}
        \\
        \cline{2-4}
        & &
        $s_t$
        & \refstepcounter{tableeqn} {\small(\thetableeqn)}\label{eq207}
        \\
	\bottomrule
	\end{tabular}
	}
	\end{center}
\end{table*}

\subsection{Feature Membership for SDD's}
This section details, in the case of SDDs, the propositional encoding
for deciding whether a subset $\fml{X} \subseteq \fml{F}$ is a $\waxp$.
Note that this encoding is not applicable to instances predicted to $\top$.
To present the constraints included in this encoding,
we need to introduce some auxiliary boolean variables and predicates.
\begin{enumerate}[nosep]
\item $s_i$, $1 \le i \le m$. $s_i$ is a selector
    such that $s_i = 1$ iff feature $i$ is included in $\fml{X}$.
    Moreover, in the context of finding one $\axp$, $s_i = 1$ also means that
    feature $i$ must be fixed to its given value $v_i$,
    while $s_i = 0$ means that feature $i$ can take any value from its domain.
\item $n^k_j$, $1 \le j \le |\mbb{C}|$ and $0 \le k \le m$. $n^k_j$ is 
    the indicator of a node $j$ of SDD $\mbb{C}$ for replica $k$.
    The indicator for the root node of $k$-th replica is $n^k_1$.
    Moreover, the semantics of  $n_j^k$
    is $n_j^k = 1$ iff the sub-SDD rooted at node $j$ in $k$-th replica
    is consistent, otherwise inconsistent.
\item
    $\ter(j) = 1$ if the node $j$ is a terminal node.
\item
    $\ele(j) = 1$ if the node $j$ is an element.
\item
    $\dec(j) = 1$ if the node $j$ is a decision node.
\item
    $\lfeat(j,i) = 1$ if the terminal node $j$ labeled with feature $i$.
\item
    $\satisfy(\lval(j),v_i) = 1$
    if for terminal node $j$, its the literal on feature $i$ is satisfied by the value $v_i$.
\end{enumerate}
The encoding is summarized in \autoref{tab:enc_sdd}.
%
%
As literals are terminal SDDs, the values of the selector variables
only affect the values of the indicator variables of terminal nodes. 
%
Constraint \eqref{eq101} states that
for any terminal node $j$ whose literal is consistent with
the given instance, its indicator $n^k_j$ is always consistent
regardless the value of $s_i$.
On the contrary,
constraint \eqref{eq103} states that
for any terminal node $j$ whose literal is inconsistent with
the given instance, its indicator $n^k_j$ is consistent iff feature $i$ is not picked,
in other words, feature $i$ can take any value.
Because replica $k$ ($k > 0$) is used to check the necessity of
including feature $k$ in $\fml{X}$,
we assume the value of the local copy of selector $s_k$ is 0 in replica $k$.
In this case, as defined in constraint \eqref{eq102}, 
even though terminal node $j$ labeled feature $k$ has a literal
that is inconsistent with the given instance, its indicator $n^k_j$ is
consistent.
Constraint \eqref{eq104} defines
the indicator for an arbitrary decision node $j$.
Constraint \eqref{eq105} defines
the indicator for an arbitrary element node $j$
(this constraint will be simplified when the sub is $\top$ or $\bot$).
Together, these constraints declare how the consistency is propagated
through the entire SDD.
Constraint \eqref{eq106} states that
the prediction of the SDD classifier $\mbb{C}$ remains $\bot$
since the selected features form a $\waxp$.
Constraint \eqref{eq107} states that
if feature $i$ is selected, then removing it
will change the prediction of $\mbb{C}$.
Finally, constraint \eqref{eq108} indicates that feature $t$ must be
included in $\fml{X}$.

\begin{example}
	For the SDD in \autoref{fig:sdd_example},
	we summarize the propositional encoding for deciding whether
        there is an $\axp$ containing feature \tbf{M}ale.
	We have selectors $\mbf{s} = \{s_P, s_Y, s_M, s_W\}$,
	If \osm is adopted, then the encoding is as follows
	(otherwise if \tsm is adopted, then formulas 0. and 3. are enough
	to check the existence of a $\waxp$):
	\begin{enumerate}[nosep]
		\item[0.] $
				(n^0_1 \lequiv n^0_2 \lor n^0_3 \lor n^0_4) \land
				(n^0_2 \lequiv n^0_5) \land
				(n^0_3 \lequiv n^0_6) \land
				(n^0_4 \lequiv n^0_7) \land
				(n^0_5 \lequiv n^0_8) \land
				(n^0_6 \lequiv n^0_{10}) \land
				(n^0_7 \lequiv n^0_{11}) \land
				(n^0_8 \lequiv \neg s_P) \land
				(n^0_{10} \lequiv \neg s_P \land \neg s_Y) \land
				(n^0_{11} \lequiv \neg s_M) \land
				(\neg n^0_1) \land (s_M)
			$
		\item $
				(n^1_1 \lequiv n^1_2 \lor n^1_3 \lor n^1_4) \land
				(n^1_2 \lequiv n^1_5) \land
				(n^1_3 \lequiv n^1_6) \land
				(n^1_4 \lequiv n^1_7) \land
				(n^1_5 \lequiv n^1_8) \land
				(n^1_6 \lequiv n^1_{10}) \land
				(n^1_7 \lequiv n^1_{11}) \land
				(n^1_8 \lequiv \neg \bot) \land
				(n^1_{10} \lequiv \neg \bot \land \neg s_Y) \land
				(n^1_{11} \lequiv \neg s_M) \land
				(s_P \lequiv n^1_1)
			$
		\item $
				(n^2_1 \lequiv n^2_2 \lor n^2_3 \lor n^2_4) \land
				(n^2_2 \lequiv n^2_5) \land
				(n^2_3 \lequiv n^2_6) \land
				(n^2_4 \lequiv n^2_7) \land
				(n^2_5 \lequiv n^2_8) \land
				(n^2_6 \lequiv n^2_{10}) \land
				(n^2_7 \lequiv n^2_{11}) \land
				(n^2_8 \lequiv \neg s_P) \land
				(n^2_{10} \lequiv \neg s_P \land \neg \bot) \land
				(n^2_{11} \lequiv \neg s_M) \land
				(s_Y \lequiv n^2_1)
			$
		\item $
				(n^3_1 \lequiv n^3_2 \lor n^3_3 \lor n^3_4) \land
				(n^3_2 \lequiv n^3_5) \land
				(n^3_3 \lequiv n^3_6) \land
				(n^3_4 \lequiv n^3_7) \land
				(n^3_5 \lequiv n^3_8) \land
				(n^3_6 \lequiv n^3_{10}) \land
				(n^3_7 \lequiv n^3_{11}) \land
				(n^3_8 \lequiv \neg s_P) \land
				(n^3_{10} \lequiv \neg s_P \land \neg s_Y) \land
				(n^3_{11} \lequiv \neg \bot) \land
				(s_M \lequiv n^3_1)
			$
		\item $
				(n^4_1 \lequiv n^4_2 \lor n^4_3 \lor n^4_4) \land
				(n^4_2 \lequiv n^4_5) \land
				(n^4_3 \lequiv n^4_6) \land
				(n^4_4 \lequiv n^4_7) \land
				(n^4_5 \lequiv n^4_8) \land
				(n^4_6 \lequiv n^4_{10}) \land
				(n^4_7 \lequiv n^4_{11}) \land
				(n^4_8 \lequiv \neg s_P) \land
				(n^4_{10} \lequiv \neg s_P \land \neg s_Y) \land
				(n^4_{11} \lequiv \neg s_M) \land
				(s_W \lequiv n^4_1)
			$
	\end{enumerate}
	%
	Solving these formulas,
	we find that for applicant Ella, there is an $\axp$ \{$\neg$ P, $\neg$ M\} containing feature \tbf{M}ale,
	so the classifier is biased.
\end{example}

\subsection{Feature Membership for XpG's}
Similarly to the previous section, this section details the
propositional encoding for deciding whether a subset $\fml{X} \subseteq \fml{F}$ is a $\waxp$,
but considers instead the case of XpG's.
The encoding is based on the evaluation function $\sigma_{\fml{D}}$.
The boolean variables $s_i$ of XpG's also play the role of selectors,
namely, $s_i = 1$ if feature $i$ is included in $\fml{X}$
(meanwhile, $s_i = 1$ also means that feature $i$ must be fixed to its given value $v_i$).
%
%

All the constraints are summarized in \autoref{tab:enc_xpg}.
Moreover, to simplify the encoding,
for an arbitrary node $j$, we replace the notation of 
its auxiliary \emph{activation} function
$\varepsilon(\mbf{s},j)$ by $n_j$ (i.e. $\varepsilon(\mbf{s},j) \lequiv n_j$)
and omit the assignment $\mbf{s}$ to $S$.
Constraints \eqref{eq201}, \eqref{eq203} and \eqref{eq204}
together form the encoding of an evaluation function $\sigma_{\fml{D}}$.
Replica $k$ ($k > 0$) is used to check feature $k$.
Thus for a non-terminal node $j$ of this replica $k$,
its auxiliary \emph{activation} function is defined as constraint \eqref{eq202}.
Similar to the encoding for SDDs,
constraint \eqref{eq205} states that
the prediction of the original classifier $\mbb{C}$ remains unchanged.
Constraint \eqref{eq206} states that
if feature $i$ is selected, then removing it
will change the prediction of $\mbb{C}$.
Finally, constraint \eqref{eq207} indicates that feature $t$ must be
included in $\fml{X}$.

\begin{example}
	For the XpG in \autoref{fig:xpg_example},
	we summarize the propositional encoding for deciding whether
        there is an $\axp$ containing feature \tbf{M}ale.
	We have selectors $\mbf{s} = \{s_P, s_Y, s_M, s_W\}$.
	If \osm is adopted, then the encoding is as follow:
	\begin{enumerate}[nosep]
		\item[0.] $
				[n^0_1] \land
				[n^0_2 \lequiv n^0_1] \land
				[n^0_3 \lequiv n^0_1 \land \neg s_P] \land
				[n^0_4 \lequiv (n^0_2 \land \neg s_M) \lor (n^0_3 \land \neg s_Y)] \land
				[n^0_6 \lequiv n^0_3 \lor n^0_4] \land
				[\sigma^0_\fml{D} \lequiv \neg n^0_6] \land [\sigma^0_\fml{D}] \land [s_M]
			$
		\item $
				[n^1_1] \land
				[n^1_2 \lequiv n^1_1] \land
				[n^1_3 \lequiv n^1_1 \land \neg \bot] \land
				[n^1_4 \lequiv (n^1_2 \land \neg s_M) \lor (n^1_3 \land \neg s_Y)] \land
				[n^1_6 \lequiv n^1_3 \lor n^1_4] \land
				[\sigma^1_\fml{D} \lequiv \neg n^1_6] \land [s_P \lequiv \neg \sigma^1_\fml{D}]
			$
		\item $
				[n^2_1] \land
				[n^2_2 \lequiv n^2_1] \land
				[n^2_3 \lequiv n^2_1 \land \neg s_P] \land
				[n^2_4 \lequiv (n^2_2 \land \neg s_M) \lor (n^2_3 \land \neg \bot)] \land
				[n^2_6 \lequiv n^2_3 \lor n^2_4] \land
				[\sigma^2_\fml{D} \lequiv \neg n^2_6] \land [s_Y \lequiv \neg \sigma^2_\fml{D}]
			$
		\item $
				[n^3_1] \land
				[n^3_2 \lequiv n^3_1] \land
				[n^3_3 \lequiv n^3_1 \land \neg s_P] \land
				[n^3_4 \lequiv (n^3_2 \land \neg \bot) \lor (n^3_3 \land \neg s_Y)] \land
				[n^3_6 \lequiv n^3_3 \lor n^3_4] \land
				[\sigma^3_\fml{D} \lequiv \neg n^3_6] \land [s_M \lequiv \neg \sigma^3_\fml{D}]
			$
		\item $
				[n^4_1] \land
				[n^4_2 \lequiv n^4_1] \land
				[n^4_3 \lequiv n^4_1 \land \neg s_P] \land
				[n^4_4 \lequiv (n^4_2 \land \neg s_M) \lor (n^4_3 \land \neg s_Y)] \land
				[n^4_6 \lequiv n^4_3 \lor n^4_4] \land
				[\sigma^4_\fml{D} \lequiv \neg n^4_6] \land [s_W \lequiv \neg \sigma^4_\fml{D}]
			$
	\end{enumerate}
	Likewise, solving these formulas will return us an $\axp$ \{$\neg$ P, $\neg$ M\} containing feature \tbf{M}ale.
\end{example}

\begin{table*}[h]
\centering
\footnotesize
\setlength{\tabcolsep}{3pt}
\begin{tabular}{cccc|cccccc|ccccc}
\multirow{3}{*}{\tbf{Name}} & \multirow{3}{*}{\tbf{\#TI}} & \multicolumn{2}{c|}{\multirow{2}{*}{\tbf{SDD}}} & \multicolumn{6}{c|}{\osm}                                                                                                          & \multicolumn{5}{c}{\tsm}                                                                       \\
                               &                                & \multicolumn{2}{c|}{}        & \multirow{2}{*}{\tbf{Succ (Test)}} & \multirow{2}{*}{\tbf{Yes\%}} & \multicolumn{2}{c}{\tbf{CNF}}          & \multicolumn{2}{c|}{\tbf{Runtime (s)}} & \multirow{2}{*}{\tbf{Yes\%}} & \multicolumn{2}{c}{\tbf{CNF}}          & \multicolumn{2}{c}{\tbf{Runtime (s)}} \\
                               &                                & \tbf{\#F}                 & \tbf{\#N}        &                                  &                                 & \tbf{Avg. \#var} & \tbf{Avg. \#cls} & \tbf{Max}      & \tbf{Avg.}     &                                 & \tbf{Avg. \#var} & \tbf{Avg. \#cls} & \tbf{Max}     & \tbf{Avg.}     \\ \hline
s1196                          & 100                            & 560                          & 2230                & 100                              & 63                              & 493536              & 576643              & 8.3               & 6.5               & 63                              & 3980                & 5579                & 0.1              & 0.1               \\
s1423                          & 100                            & 748                          & 3493                & 100                              & 60                              & 1285480             & 1538868             & 30.6              & 19.1              & 60                              & 6636                & 9831                & 0.2              & 0.1               \\
s1488                          & 100                            & 667                          & 3248                & 100                              & 65                              & 1118950             & 1439323             & 17.0              & 15.5              & 65                              & 6201                & 9458                & 0.1              & 0.1               \\
s1494                          & 100                            & 661                          & 2644                & 100                              & 54                              & 683740              & 675799              & 10.1              & 8.9               & 54                              & 4715                & 6184                & 0.1              & 0.1               \\
s400                           & 100                            & 189                          & 2150                & 100                              & 95                              & 532189              & 1362872             & 84.6              & 24.6              & 95                              & 5789                & 14348               & 0.6              & 0.2               \\
s420.1                         & 100                            & 252                          & 2525                & 100                              & 100                             & 750144              & 1789977             & 48.9              & 22.6              & 100                             & 6180                & 14152               & 0.2              & 0.2               \\
s444                           & 100                            & 205                          & 2586                & 100                              & 96                              & 731587              & 1914949             & 170.2             & 48.3              & 96                              & 7321                & 18726               & 2.9              & 0.2               \\
s510                           & 100                            & 236                          & 4180                & 100                              & 100                             & 1290701             & 3196895             & 108.9             & 39.4              & 100                             & 11126               & 26980               & 0.4              & 0.3               \\ \hline
s526                           & 100                            & 217                          & 3451                & 100                              & 100                             & 1019367             & 2705600             & 309.8             & 82.6              & 100                             & 9567                & 24824               & 0.4              & 0.2               \\
s526n                          & 100                            & 218                          & 5547                & 99 (100)                      & 99                              & 2149046             & 7092974             & 1800              & 338.6             & 100                             & 19842               & 64778               & 1.0              & 0.6               \\
s641                           & 100                            & 433                          & 2044                & 100                              & 58                              & 441190              & 504836              & 8.9               & 6.2               & 58                              & 3841                & 5576                & 0.1              & 0.0               \\
s713                           & 100                            & 447                          & 2050                & 100                              & 56                              & 470931              & 517371              & 8.8               & 6.7               & 56                              & 3989                & 5695                & 0.1              & 0.1               \\
s820                           & 100                            & 312                          & 1409                & 100                              & 60                              & 213612              & 252735              & 3.8               & 2.9               & 60                              & 2712                & 3970                & 0.1              & 0.0               \\
s832                           & 100                            & 310                          & 1420                & 100                              & 51                              & 213123              & 212591              & 3.6               & 2.9               & 51                              & 2706                & 3744                & 0.1              & 0.0               \\
s838.1                         & 100                            & 512                          & 5341                & 100                              & 100                             & 3144176             & 7526738             & 1390.6            & 212.5             & 100                             & 12768               & 29346               & 3.5              & 1.0               \\
s953                           & 100                            & 417                          & 1692                & 100                              & 39                              & 285832              & 210860              & 4.2               & 3.5               & 39                              & 3059                & 3694                & 0.1              & 0.0               \\
s344                           & 100                            & 184                          & 2581                & 100                              & 100                             & 803454              & 2229067             & 322.9             & 71.1              & 100                             & 8868                & 24100               & 0.3              & 0.2               \\
s499                           & 100                            & 175                          & 2282                & 100                              & 100                             & 507407              & 1380722             & 25.7              & 11.8              & 100                             & 5939                & 15692               & 0.3              & 0.2               \\
s635                           & 100                            & 320                          & 2972                & 100                              & 100                             & 1217552             & 3042761             & 183.5             & 45.8              & 100                             & 7904                & 18960               & 1.3              & 0.4               \\
s938                           & 100                            & 512                          & 5615                & 100                              & 99                              & 3258575             & 7862753             & 443.3             & 159.5             & 99                              & 13214               & 30656               & 2.8              & 1.1               \\
s967                           & 100                            & 416                          & 2292                & 100                              & 72                              & 555693              & 839976              & 10.9              & 8.3               & 72                              & 4576                & 7666                & 0.1              & 0.1               \\
s991                           & 100                            & 603                          & 2799                & 100                              & 74                              & 1001100             & 1511707             & 27.4              & 16.3              & 74                              & 5613                & 9100                & 0.2              & 0.1               \\ \hline
Accidents                      & 100                            & 415                          & 8863                & 23 (32)                 & 23                              & 5428799 & 16280994               & \multicolumn{2}{c|}{Timeout} & 97                              & 26513               & 78276               & 56.4             & 3.5               \\
Audio                          & 100                            & 272                          & 7224                & 23 (34)               & 23                              & 4214846 & 13782407            & \multicolumn{2}{c|}{Timeout} & 88                              & 31148               & 100972              & 663.1            & 22.0              \\
DNA                            & 100                            & 513                          & 8570                & 5 (18)                 & 5                               & 7361507 & 23460504             & \multicolumn{2}{c|}{Timeout} & 91                              & 29155               & 91288               & 86.3             & 11.0              \\
Jester                         & 100                            & 254                          & 7857                & 19 (35)          & 19                              & 4557614 & 15492017             & \multicolumn{2}{c|}{Timeout} & 85                              & 35998               & 121508              & 362.1            & 22.7              \\
KDD                            & 100                            & 306                          & 8109                & 31 (38)                    & 31                   & 4006042 & 12813875         & \multicolumn{2}{c|}{Timeout} & 99                              & 26402               & 83480               & 111.2            & 2.8               \\
Mushrooms                      & 100                            & 248                          & 7096                & 53 (59)          & 53                              & 2941685 & 10222697    & \multicolumn{2}{c|}{Timeout} & 91                              & 23874               & 82112               & 266.3            & 15.8              \\
Netflix                        & 100                            & 292                          & 7039                & 34 (41)           & 34                            & 3696194 & 12206675  & \multicolumn{2}{c|}{Timeout} & 94                              & 25520               & 83324               & 105.7            & 4.2               \\
NLTCS                      & 100                            & 183                          & 6661                & 100                              & 100                             & 1806511             & 5381266             & 816.4             & 166.4             & 100                             & 19817               & 58494               & 1.4              & 0.5               \\
Plants                         & 100                            & 244                          & 6724                & 20 (33)         & 20                             & 3076464 & 10385552  & \multicolumn{2}{c|}{Timeout} & 97                              & 25356               & 84782               & 950.7            & 20.6              \\
RCV-1                           & 100                            & 410                          & 9472                &  10 (24)                        & 10              & 6787664 & 21063341           & \multicolumn{2}{c|}{Timeout} & 90                              & 33438               & 102500              & 153.6            & 11.2              \\
Retail                         & 100                            & 341                          & 3704                & 100                              & 87                              & 1754801             & 4846142             & 909.8             & 207.6             & 87                              & 10601               & 28342               & 1.8              & 1.1              
\end{tabular}
\caption{Solving FMP for SDDs with two methods.
First column reports the name of each test case.
Column {\bf \#TI} reports the number of tested instances.
Sub-column {\bf \#F} reports the number of features that appear in the SDD. 
Sub-column {\bf \#N} reports the number of nodes of a SDD.
Column {\bf Succ (Test)} shows the number of solved queries,
inside the parentheses is the number of tested queries.
Column {\bf Yes\%} shows the percentage of answering `Yes' to the queries.
Sub-Columns {\bf Avg. \#var} and {\bf Avg. \#cls} show, respectively,
the average number of variables and clauses in a CNF encoding.
Sub-columns {\bf Max} and {\bf Avg.} 
reports, respectively, maximal and average time in seconds
for answering a query.
}
\label{tab:sdd}
\end{table*}
\begin{figure*}[h]
\centering
\includegraphics[width=0.55\textwidth]{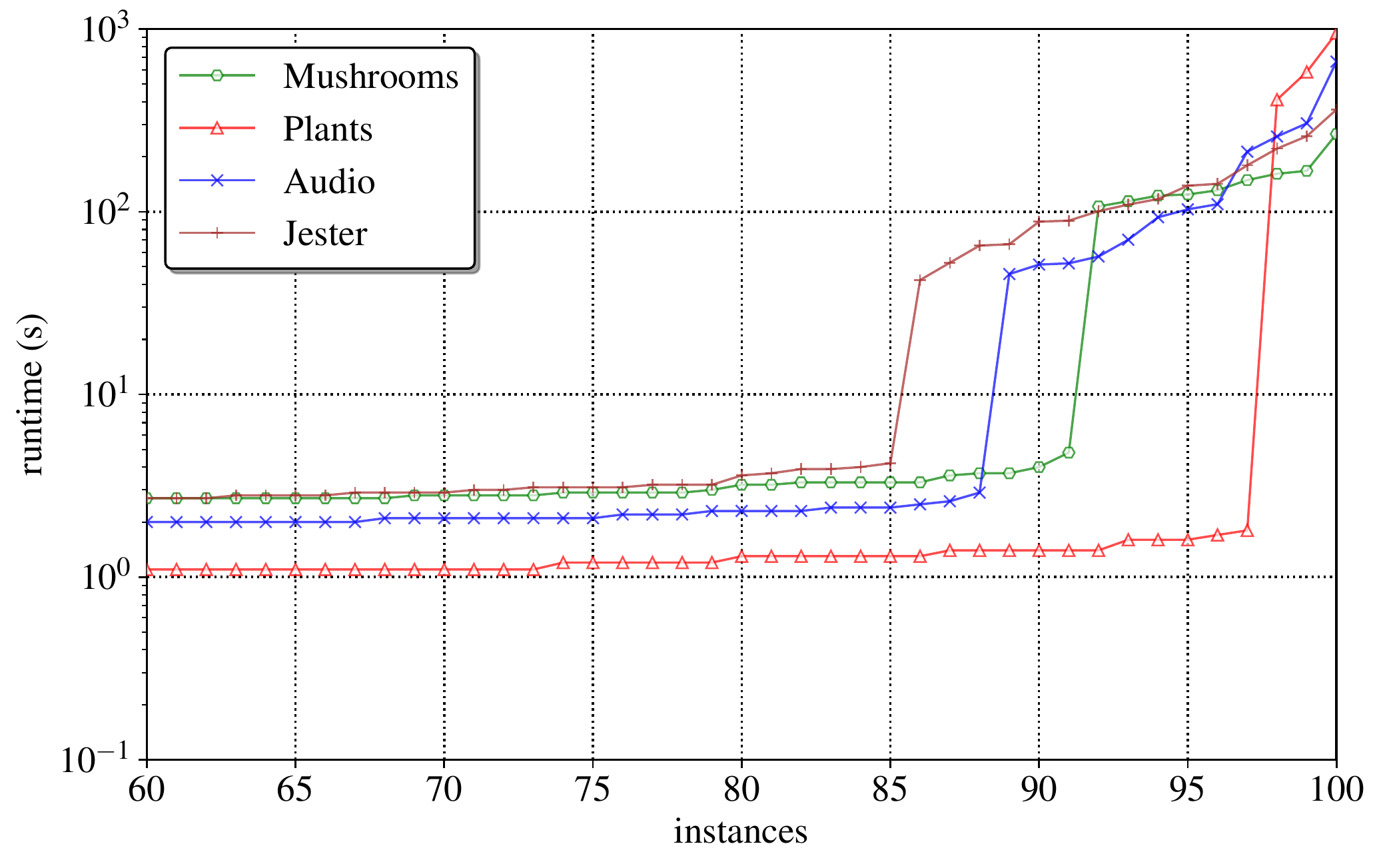}
\caption{Running times of Audio, Jester, Mushrooms and Plants.}
\label{fig:runtime}
\end{figure*}
\begin{table*}[h]
\centering
\footnotesize
\setlength{\tabcolsep}{3pt}
\begin{tabular}{cccc|ccccc|ccccc}
\multirow{3}{*}{\tbf{Name}} & \multirow{3}{*}{\tbf{\#TI}} & \multicolumn{2}{c|}{\multirow{2}{*}{\tbf{XpG}}} & \multicolumn{5}{c|}{\osm}                                                                                    & \multicolumn{5}{c}{\tsm}                                                                  \\
                               &                                & \multicolumn{2}{c|}{}                              & \multirow{2}{*}{\tbf{Yes\%}} & \multicolumn{2}{c}{\tbf{CNF}}          & \multicolumn{2}{c|}{\tbf{Runtime (s)}} & \multirow{2}{*}{\tbf{Yes\%}} & \multicolumn{2}{c}{\tbf{CNF}} & \multicolumn{2}{c}{\tbf{Runtime (s)}} \\
                               &                                & \tbf{\#F}                    & \tbf{\#N}                      &                                  & \tbf{Avg. \#var} & \tbf{Avg. \#cls} & \tbf{Max}      & \tbf{Avg.}     &                         & \tbf{Avg. \#var} & \tbf{Avg. \#cls} & \tbf{Max}          & \tbf{Avg.}         \\ \hline
adult                          & 100                            & 13                      & 299                      & 59                               & 5851                & 6431                & 0.2               & 0.1               & 59                      & 847                 & 1446       & 0.1                   & 0.0          \\
letter                         & 100                            & 16                      & 949                      & 46                               & 31620               & 30424               & 0.8               & 0.4               & 46                      & 3734                & 6153       & 0.2                   & 0.1          \\
mfeat\_fourier                 & 100                            & 42                      & 265                      & 18                               & 21676               & 21666               & 0.3               & 0.2               & 18                      & 1048                & 1709       & 0.1                   & 0.0          \\
mfeat\_karhunen                & 100                            & 33                      & 281                      & 34                               & 18202               & 18099               & 0.3               & 0.2               & 34                      & 1102                & 1833       & 0.1                   & 0.0          \\
mfeat\_zernike                 & 100                            & 32                      & 299                      & 29                               & 18824               & 22363               & 0.3               & 0.2               & 29                      & 1171                & 2010       & 0.0                   & 0.0          \\
satimage                       & 100                            & 36                      & 355                      & 43                               & 24045               & 31603               & 0.3               & 0.3               & 43                      & 1334                & 2429       & 0.0                   & 0.0          \\
twonorm                        & 100                            & 20                      & 439                      & 87                               & 13848               & 22148               & 0.2               & 0.2               & 87                      & 1337                & 2623       & 0.1                   & 0.0          \\
waveform\_40                   & 100                            & 38                      & 431                      & 42                               & 28100               & 41340               & 0.4               & 0.3               & 42                      & 1477                & 2742       & 0.1                   & 0.0          \\ \hline
flat30-3                       & 100                            & 90                      & 10012                    & 100                              & 1876874             & 3372082             & 49.9              & 31.1              & 100                     & 41338               & 94264      & 0.9                   & 0.7          \\
flat30-29                      & 100                            & 90                      & 8745                     & 100                              & 1627989             & 2741722             & 42.2              & 25.7              & 100                     & 35868               & 81473      & 0.6                   & 0.6          \\
flat30-33                      & 100                            & 90                      & 9004                     & 100                              & 1690324             & 2915374             & 99.8              & 29.6              & 100                     & 37238               & 84977      & 1.0                   & 0.6          \\
flat30-36                      & 100                            & 90                      & 13015                    & 100                              & 2452267             & 4240009             & 73.0              & 42.5              & 100                     & 53984               & 123504     & 1.1                   & 0.9          \\
flat30-37                      & 100                            & 90                      & 15681                    & 100                              & 2949309             & 5339047             & 81.4              & 50.8              & 100                     & 64908               & 148503     & 1.8                   & 1.1          \\
flat30-56                      & 100                            & 90                      & 12597                    & 100                              & 2357809             & 4108588             & 73.8              & 41.4              & 100                     & 51908               & 118295     & 1.1                   & 0.9          \\
flat30-58                      & 100                            & 90                      & 7724                     & 100                              & 1448537             & 2617239             & 50.3              & 25.0              & 100                     & 31924               & 72799      & 0.7                   & 0.5          \\
flat30-61                      & 100                            & 90                      & 10076                    & 100                              & 1879058             & 3303417             & 48.3              & 30.1              & 100                     & 41386               & 94174      & 0.9                   & 0.7          \\
flat30-66                      & 100                            & 90                      & 10686                    & 100                              & 2020927             & 3672495             & 74.6              & 35.3              & 100                     & 44504               & 101979     & 1.1                   & 0.8          \\
flat30-71                      & 100                            & 90                      & 11594                    & 100                              & 2166891             & 3832417             & 106.0             & 41.4              & 100                     & 47712               & 108707     & 1.0                   & 0.8          \\
flat30-81                      & 100                            & 90                      & 14464                    & 100                              & 2719079             & 4743534             & 67.9              & 44.0              & 100                     & 59848               & 136815     & 1.5                   & 1.0          \\
flat30-86                      & 100                            & 90                      & 7930                     & 99                               & 1496949             & 2607069             & 387.9             & 30.0              & 99                      & 32988               & 75427      & 3.7                   & 0.6          \\
flat30-88                      & 100                            & 90                      & 21816                    & 100                              & 4101824             & 7201944             & 101.6             & 67.1              & 100                     & 90238               & 206328     & 2.6                   & 1.6          \\
flat30-96                      & 100                            & 90                      & 9265                     & 100                              & 1747199             & 3173768             & 81.4              & 32.1              & 100                     & 38488               & 87948      & 1.0                   & 0.7         
\end{tabular}
\caption{Solving FMP for XpGs with two methods.
The columns hold the same meaning as described in the caption of \cref{tab:sdd}.
}
\label{tab:xpg}
\end{table*}

\section{Preliminary Experimental Results} \label{sec:res}

This section presents preliminary experimental results on assessing
the practical efficiency of the proposed methods.
The experiments were performed on a MacBook Pro with a 6-Core Intel
Core~i7 2.6~GHz processor with 16~GByte RAM, running macOS Monterey.

\paragraph{Classifiers and Benchmarks.}
We consider SDD, DT, and OBDD classifiers 
(DTs and OBDDs were then mapped into XpGs).
For SDDs,
we selected 16 circuits from ISCAS89 suite,
6 circuits from ISCAS93 suite~\footnote{http://www.cril.univ-artois.fr/KC/benchmarks.html},
and 11 datasets from
Density Estimation Benchmark Datasets\footnote{https://github.com/UCLA-StarAI/Density-Estimation-Datasets}.
~\cite{lowd2010learning,van2012markov,larochelle2011neural}.
22 circuits were compiled into SDDs by using the
well-known SDD package\footnote{http://reasoning.cs.ucla.edu/sdd/}.
11 datasets were used to learn SDD via using
LearnSDD\footnote{https://github.com/ML-KULeuven/LearnSDD}~\cite{bekker2015tractable}
(with parameter \textit{maxEdges=20000}).
The obtained SDDs were used as binary classifiers (albeit the selected
circuits/datasets might not originally target classification tasks.)
For XpG,
we selected 8 classification datasets from the
Penn Machine Learning Benchmarks~\cite{Olson2017PMLB}, 
and 14 test cases from a graph colouring problems benchmark
flat-30-60~\footnote{https://www.cs.ubc.ca/~hoos/SATLIB/benchm.html}
(the rest test cases are filtered out since their size are below 7500 nodes).
8 datasets were used to learn DTs by using Orange3~\cite{JMLR:demsar13a}.
14 test cases were compiled into OBDDs by using
\texttt{dd}~\footnote{https://github.com/tulip-control/dd} package
which integrated well-known
CUDD~\footnote{https://github.com/ivmai/cudd}
~\cite{somenzi2012cudd}
package.
%

\paragraph{Prototype implementation.}
A prototype implementation of the proposed
approach was implemented in Python\footnote{https://github.com/XuanxiangHuang/fmp-experiments}.
The PySAT toolkit~\cite{imms-sat18} was employed to perform feature
membership encoding, 
and called Glucose 4~\cite{audemard2018glucose} SAT solver.
SDD/XpG models were loaded by using 
PySDD\footnote{https://github.com/wannesm/PySDD}/xpg\footnote{https://github.com/yizza91/xpg}
package.

\paragraph{Experimental procedure.}
To assess the efficiency of deciding feature membership, 
and for each classifier, 100 test instances
were randomly generated/selected.
For SDDs, all tested instances have prediction $\bot$.
(We didn't pick instances predicted to class $\top$
as this requires the compilation of a new classifier which may have different size).
Besides, for each instance, we randomly picked a feature
appearing in the model.
Hence for each SDD/XpG, we solved 100 queries.
The time for deciding FMP was limited to 1800 seconds.
And the time for finishing 100 queries was limited to 10 hours,
this means the average time for deciding FMP cannot exceed
6 minutes.
Note that for SDDs learned from LearnSDD, the reported number of features
includes both original features and generated features
(e.g. for \textit{Audio} the original number of features is 100).
Also note that PySDD offers canonical SDDs whose \textit{conditioning}
may take exponential time in the worst-case.
Nevertheless, this worst-case behaviour was not observed in the
experiments.

\paragraph{Results.}
Table \ref{tab:sdd} summarizes the obtained results of deciding FMP on
SDDs with two methods.
In this experiment,
it can be observed that the number of nodes of the tested SDD
is in the range of 1409 and 9472, and the number of features of tested SDD
is in the range of 175 and 748.
The \osm requires $m+1$ replicas, often leading to large CNF encodings. 
The increase on both the number of features and the number of nodes,
can results in timeouts being observed.
One observation is that the performance correlates inversely with
propositional formula size. For the \osm this is noticeable when the
number of clauses in the CNF formulas exceeds 7,000,000.
For \textit{s526n}, the \osm failed to solve all 100 queries.
For \textit{Accidents}, \textit{Audio}, \textit{DNA}, \textit{Jester}, 
\textit{KDD}, \textit{Mushrooms}, \textit{Netflix}, \textit{Plants} and \textit{RCV-1},
the \osm can only solve a small number of queries
(e.g. for DNA, only 18 queries are tested, and only 5 queries are solved,
13 queries out of 18 cannot be solved in 1800 seconds time limit,
and the rest 82 queries were not tested due to the 10 hours time limit.)

In contrast, the \tsm is much more efficient as the CNF encoding of
\tsm is much smaller (the average number of CNF clauses does not exceed
130,000).
For the SDDs compiled from 16 circuits, the \tsm successfully solve
all the queries. For any of the examples considered, the \tsm never
requires more than a few seconds to answer a query, and the average
running time is at least one order of magnitude smaller than that of
the \osm.
For the remaining SDDs, the average running time for \tsm to solve a
query is less than 25 seconds; this highlights the scalability of the
\tsm.
However, notice that for SDDs representing
\textit{Audio}, \textit{Jester}, \textit{Mushrooms} and \textit{Plants},
the largest running time for deciding FMP with the \tsm can exceed 3
minutes.
As a result, we analyzed these results in greater detail.
\autoref{fig:runtime} depicts a cactus plot showing the running time
(in seconds) of deciding FMP for these 4 datasets
(note that the runtime axis is scaled logarithmically, and the instances axis starts from 60).
As can be observed, for each of dataset, around 85 queries can be
solved in a few seconds. This means that the running times of the \tsm 
only exceeds a few seconds for a few concrete examples, and for a few
of the datasets considered.

Table \ref{tab:xpg} summarizes the obtained results of deciding FMP on
XpGs with two methods.
No timeout occurs in this experiment.
For XpGs reduced from DTs, the running time for deciding FMP is
negligible regardless the method we adopt, this is due to the number
of nodes of each tree XpG is small.
For XpGs reduced from OBDDs, despite the size of each XpG is not small,
using \tsm only takes maximal few seconds to solve a query.
Furthermore, even though the average running time of the \osm is not
prohibitive, the \tsm still outperforms the \osm by at least one order
of magnitude.

\section{Conclusions} \label{sec:conc}
This paper proves that, for classifiers for which one explanation can
be computed in polynomial time, then the feature membership problem is
in NP.
Furthermore, for SDDs and also classifiers that can be mapped to
explanation graphs (XpG's), this paper details two propositional
encodings to decide the existence of one explanation containing
desired feature.
The experiments confirm the practical efficiency and scalability of
one of the proposed encodings, both for SDDs and XpGs.

\section*{Acknowledgments}
This work was supported by the AI Interdisciplinary Institute ANITI,
funded by the French program ``Investing for the Future -- PIA3''
under Grant agreement no.\ ANR-19-PI3A-0004, and by the H2020-ICT38
project COALA ``Cognitive Assisted agile manufacturing for a Labor
force supported by trustworthy Artificial intelligence''.

\bibliographystyle{kr}

\end{document}